\documentclass[11pt]{article}
\usepackage[T1]{fontenc}
\usepackage{lmodern}
\usepackage[margin=1in]{geometry}
\usepackage{amsmath,amssymb,amsthm,mathtools}
\usepackage{booktabs,microtype}
\usepackage{graphicx}
\usepackage[table]{xcolor}
\usepackage[section]{placeins}
\usepackage[numbers,sort&compress]{natbib}
\usepackage[colorlinks=true,linkcolor=blue,citecolor=blue,urlcolor=blue]{hyperref}
\usepackage{url}
\newcommand{\mask}{\textsc{mask}}
\newcommand{\E}{\mathbb{E}}
\newcommand{\ind}{\mathbf{1}}
\newcommand{\R}{\mathbb{R}}
\newcommand{\tdlm}{\textsc{TDLM}}
\definecolor{medalgold}{HTML}{F4E7B5}
\definecolor{medalsilver}{HTML}{E6E8EB}
\definecolor{medalbronze}{HTML}{E9CFBC}

\newtheorem{proposition}{Proposition}
\title{TDLM: Adaptive Temporal Integration for Diffusion Language Models}
\author{
Bohang Sun$^{1}$
\quad Max Zhu$^{1}$
\quad Francesco Caso$^{1}$ 
\quad Christopher Irwin$^{1}$\\ 
\quad Jialin Yu$^{2,3}$\thanks{ \texttt{yu.jialin@outlook.com}.}
\quad Pietro Li\`o$^{1}$\thanks{ \texttt{pl219@cam.ac.uk}.} \\[0.75em]
{\small $^{1}$Department of Computer Science and Technology, University of Cambridge} \\
{\small $^{2}$Department of Engineering Science, University of Oxford} \\
{\small $^{3}$University of Bath}
}
\date{}

\begin{document}
\begingroup
\makeatletter
\let\standardtitlesymbol\@fnsymbol
\renewcommand{\@fnsymbol}[1]{%
  \ifnum#1=2\relax
    \textasteriskcentered\textasteriskcentered
  \else
    \standardtitlesymbol{#1}%
  \fi
}
\makeatother
\maketitle
\endgroup
\begin{abstract}
Masked diffusion language models repeatedly estimate unresolved tokens as their context changes. Standard denoisers predict from the current partially revealed sequence at each evaluation. We introduce TDLM, which reuses earlier hidden representations at the same target position in a frozen denoiser. During decoding, TDLM maintains history for masked positions in the active block and one lookahead block. We train the temporal modules on the backbone's own decoding trajectories. On LLaDA, the modules add 0.84\% of the backbone's parameter count. At fixed decoding settings, TDLM uses fewer denoiser evaluations and less time than Fast-dLLM and MetaState on all six backbone--task pairs. On GSM8K with LLaDA at a 256-token budget, TDLM with dual spatial KV caching achieves an estimated 8.96$\times$ speedup over a separately timed uncached greedy baseline and a 1.94$\times$ speedup over Fast-dLLM with dual KV caching. TDLM reaches the GSM8K accuracy--latency frontier on both LLaDA and Dream.
\end{abstract}

\section{Introduction}
\label{sec:introduction}

Masked diffusion language models generate text by predicting missing tokens in parallel and committing a subset of those predictions at each step. LLaDA, iLLaDA, and Dream demonstrate this approach at billion-parameter scale~\citep{nie2025llada,nie2026improvedlargelanguagediffusion,ye2025dream}. Nemotron-Labs-Diffusion and DiffusionGemma extend parallel generation while supporting autoregressive decoding~\citep{fu2026nemotronlabsdiffusiontrimodelanguagemodel,diffusiongemmateam2026diffusiongemmatechnicalreport}. Generation speed depends on how many predictions the decoder can reliably commit per evaluation. Confidence-based decoding controls this parallelism through a threshold~\citep{wu2025fastdllm}.

An unresolved position can remain masked through many evaluations. As neighboring tokens are revealed, each evaluation computes a new representation for that position. A standard denoiser carries forward the revealed tokens, leaving these intermediate representations unavailable to later predictions.

We introduce \tdlm{}, a temporal self-conditioning extension to a pretrained diffusion language model. At selected layers, masked positions attend to their own representations from earlier denoising steps. A learned zero-value sink gates how much each position retrieves, and narrow projections keep the adaptation small. We train the temporal modules to predict upcoming commitments in the backbone's own trajectories using cross-entropy. Causal attention evaluates sampled trajectory inputs in parallel during training. During decoding, TDLM maintains history for masked positions in the active block and one lookahead block.

At fixed decoding settings, TDLM uses fewer denoiser evaluations and less time than Fast-dLLM and MetaState on both backbones and all three tasks (Table~\ref{tab:main-results}). With a 256-token budget, TDLM delivers up to a 7.15$\times$ speedup over greedy original decoding. On LLaDA GSM8K, it raises accuracy from 78.62\% to 79.08\% at a 5.91$\times$ speedup (Appendix~\ref{app:fixed-setting-performance}). Figure~\ref{fig:gsm8k-time-accuracy} shows TDLM on the GSM8K accuracy--latency frontier for both backbones. History-length ablations and parameter-matched LoRA controls probe how access to earlier representations contributes to this trade-off.

\begin{figure}[t]
\centering
\includegraphics[width=0.8\linewidth]{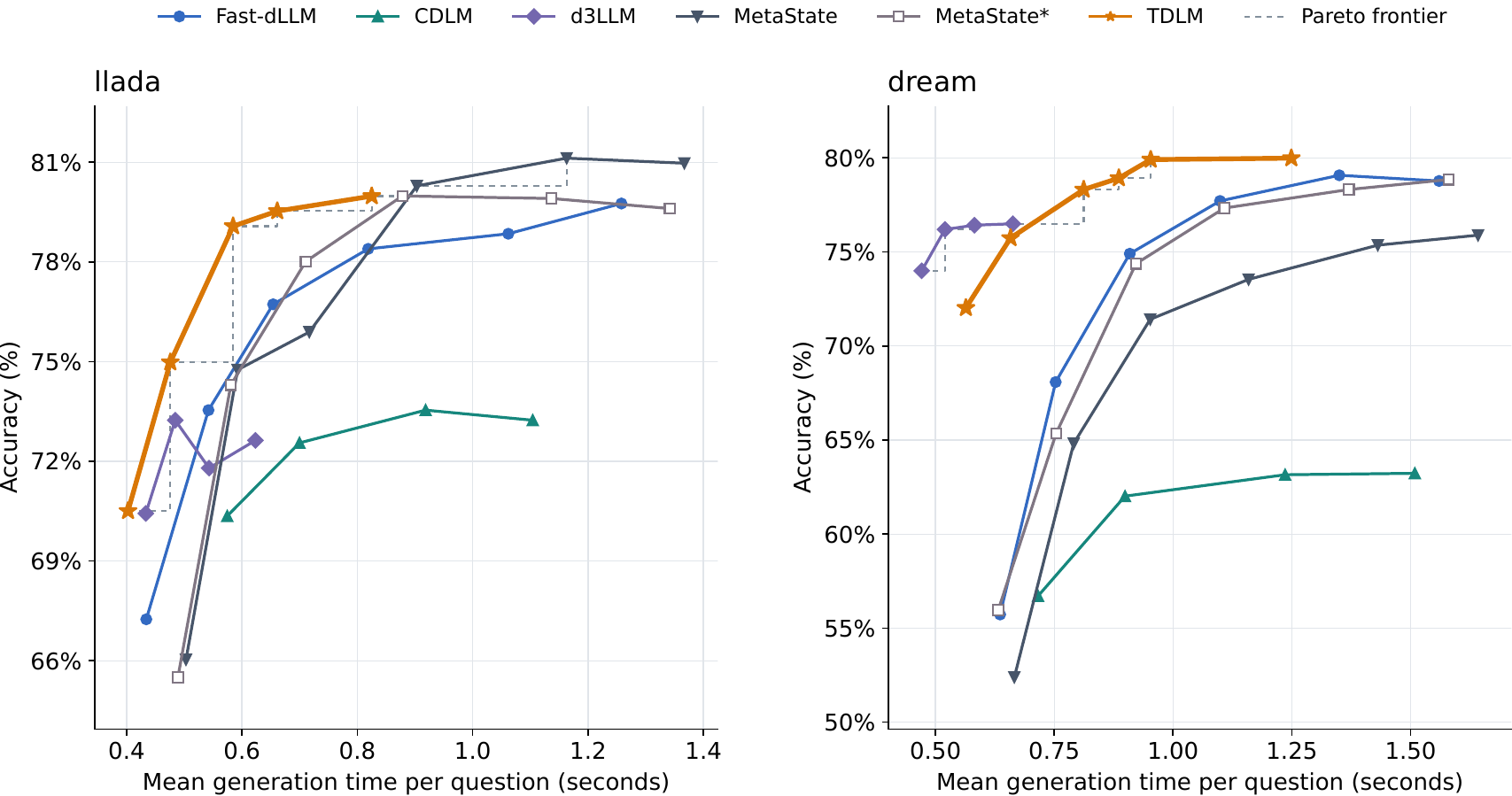}
\caption{\textbf{Accuracy--latency trade-off on GSM8K.} Threshold sweeps on LLaDA (left) and Dream (right). Dashed lines mark the Pareto frontiers. MetaState* uses the MetaState architecture trained on the trajectories collected for TDLM.}
\label{fig:gsm8k-time-accuracy}
\end{figure}

\section{Method}
\label{sec:method}

TDLM augments a frozen denoiser with temporal attention. Spatial attention exchanges information across positions in the current input. Temporal attention retrieves earlier representations at the same target position.

\subsection{Conditioning on denoising history}

Let $c$ be a prompt and $N$ the response length. Each decoding step evaluates the denoiser and reveals one or more tokens: their masks are replaced with predicted tokens, which are then kept fixed. We refer to this operation as commitment.

We call a partially revealed response sequence a \emph{state} and write it as $S_i$. Its global decoding step $s_i$ counts the commitment updates already completed. During training, $i$ indexes the selected inputs along the teacher trajectory. At inference, we evaluate every successive input. The initial sequence is $S_0=\mask^N$ at $s_0=0$.

For a masked position $p$ with retained history, TDLM extends the information available to predict its token $y_p$:
\begin{equation}
 p_\theta(y_p\mid c,S_i)
 \quad\longrightarrow\quad
 p_{\theta,\phi}(y_p\mid c,S_i,\mathcal H_{<i}),
 \label{eq:conditioning}
\end{equation}
where $\theta$ denotes the frozen backbone parameters, $\phi$ the temporal parameters, and $\mathcal H_{<i}$ the hidden representations retained from earlier evaluated inputs $S_j$, $j<i$. The input at position $p$ can remain a mask while its internal representation changes with the surrounding context. Temporal memory makes those earlier representations available to the next prediction.

\subsection{Temporal attention at masked token positions}

At a selected layer $\ell$, let $h_{\ell,p}(S_i)\in\R^d$ be the representation used to predict token $p$, after spatial attention and before the feed-forward sublayer. This is the representation at backbone position $p$ on LLaDA and $p-1$ on Dream. Define $m_{i,p}=\ind[(S_i)_p=\mask]$ and let $r_i=\sum_p(1-m_{i,p})$ count the revealed response tokens in $S_i$. Each decoding step can commit multiple tokens, so the RoPE coordinate $r_i$ measures progress in tokens.

During decoding, let $\mathcal W_i$ contain the active block and the next block, if present. The temporal mask is $g_{i,p}=m_{i,p}\ind[p\in\mathcal W_i]$ during decoding and $g_{i,p}=m_{i,p}$ during training. By default, a query at a position with $g_{i,p}=1$ reads the earlier entries indexed by $\mathcal J_p(i)=\{j<i\mid g_{j,p}=1\}$. The layer-specific history $\mathcal H_{\ell,p}^{<i}$ contains the corresponding pairs $(h_{\ell,p}(S_j),r_j)$. The residual update is
\begin{equation}
 \widetilde h_{\ell,p}(S_i)
 =h_{\ell,p}(S_i)+g_{i,p}\,\operatorname{TempAttn}_{\ell}
 \!\left(h_{\ell,p}(S_i),\mathcal H_{\ell,p}^{<i};r_i\right).
 \label{eq:temporal}
\end{equation}
$\operatorname{TempAttn}$ uses narrow multi-head projections and retrieves from the same target position and layer at earlier evaluations. Temporal rotary embeddings~\citep{su2021roformer} index queries and keys by revealed-token counts. A learned zero-value sink gates the magnitude of the history update (Appendix~\ref{app:temporal-details}).

\begin{figure}[t]
\centering
\includegraphics[width=0.95\linewidth]{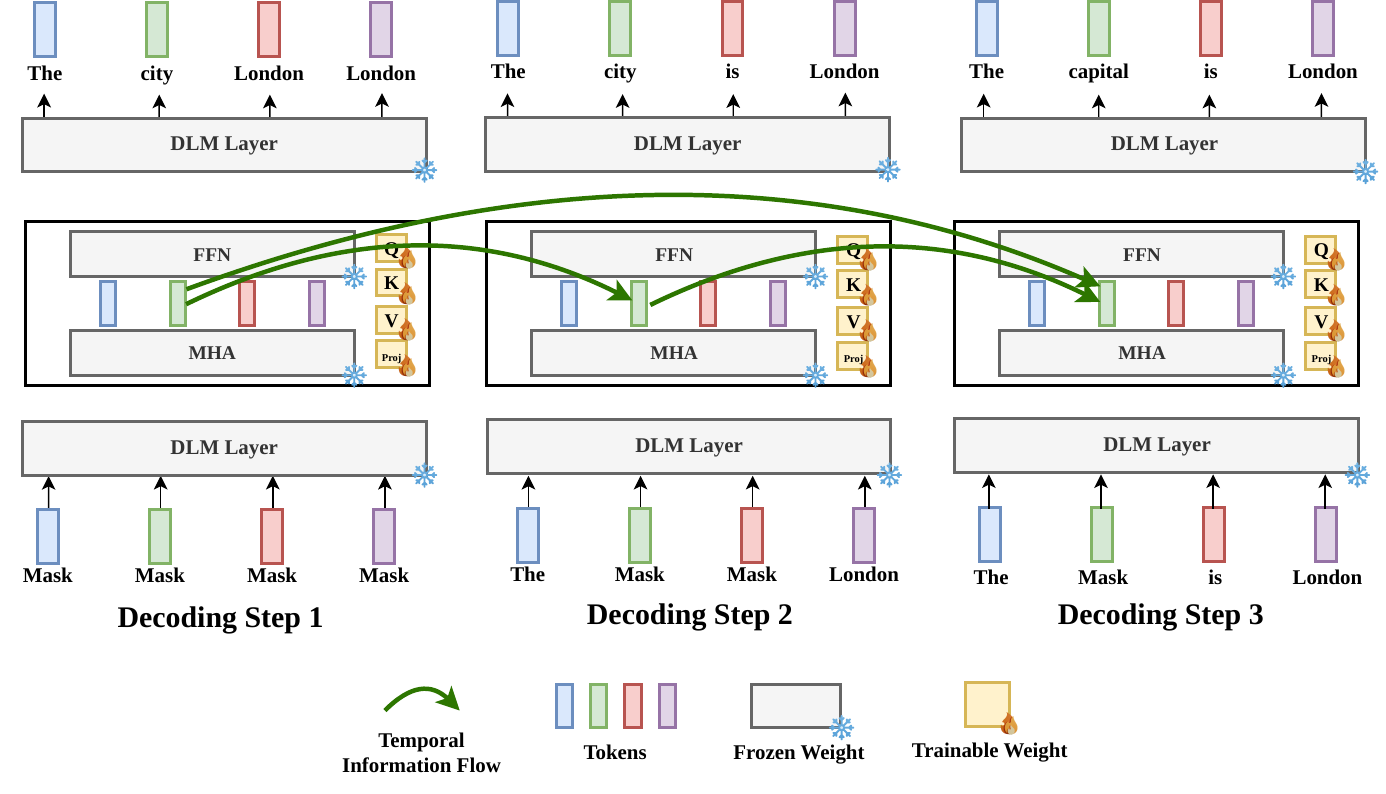}
\caption{\textbf{Reusing computation before commitment.}
Each column is a denoising evaluation, read from the partially revealed input at the bottom to the predictions before commitment at the top. Decoding Steps 1--3 use inputs $S_0,S_1,S_2$ at $s_0=0,s_1=1,s_2=2$, respectively. Predictions change at the green masked position. Green arrows retrieve earlier hidden representations at that position and layer. Yellow modules are trainable, and the backbone is frozen.}
\label{fig:tdlm-overview}
\end{figure}

The original feed-forward sublayer processes the corrected representation $\widetilde h$. Keys and values are formed from $h$ before the same layer's temporal correction, a property used by parallel trajectory training below. Empty history produces zero correction, and a zero-initialized output projection preserves the pretrained denoiser at initialization. The precise statements and proofs are in Appendix~\ref{app:properties}.

\subsection{Trajectory self-distillation}

For each prompt, the frozen backbone generates a teacher response in $T$ decoding steps. We record the final tokens $y_p$ and the step $a_p$ at which each token is revealed. These records let us reconstruct the input at any selected step $s_i$: tokens revealed by that step are visible, and later tokens remain masked,
\begin{equation}
 (S_i)_p=
 \begin{cases}y_p,&a_p\leq s_i,\\\mask,&a_p>s_i.\end{cases}
 \label{eq:state}
\end{equation}
Hidden representations are computed during training from these reconstructed inputs.

We always include $s_0=0$, sample $K-1$ distinct steps uniformly from $\{1,\ldots,T-1\}$, and sort them. We also set $s_K=T$ to mark the end of the trajectory:
\begin{equation*}
 0=s_0<s_1<\cdots<s_{K-1}<s_K=T.
\end{equation*}
We evaluate $S_0,\ldots,S_{K-1}$, with $S_0$ seeding temporal memory. For each later input $S_i$, the supervision window is $(s_i,s_{i+1}]$. We apply token-level cross-entropy $-\log\pi_{i,p}(y_p)$ to each teacher token committed within that window, where $\pi_{i,p}$ is TDLM's predicted token distribution at position $p$ for $S_i$. Figure~\ref{fig:training-stage} illustrates the sampled inputs and targets.

\begin{figure}[t]
\centering
\includegraphics[width=0.8\linewidth]{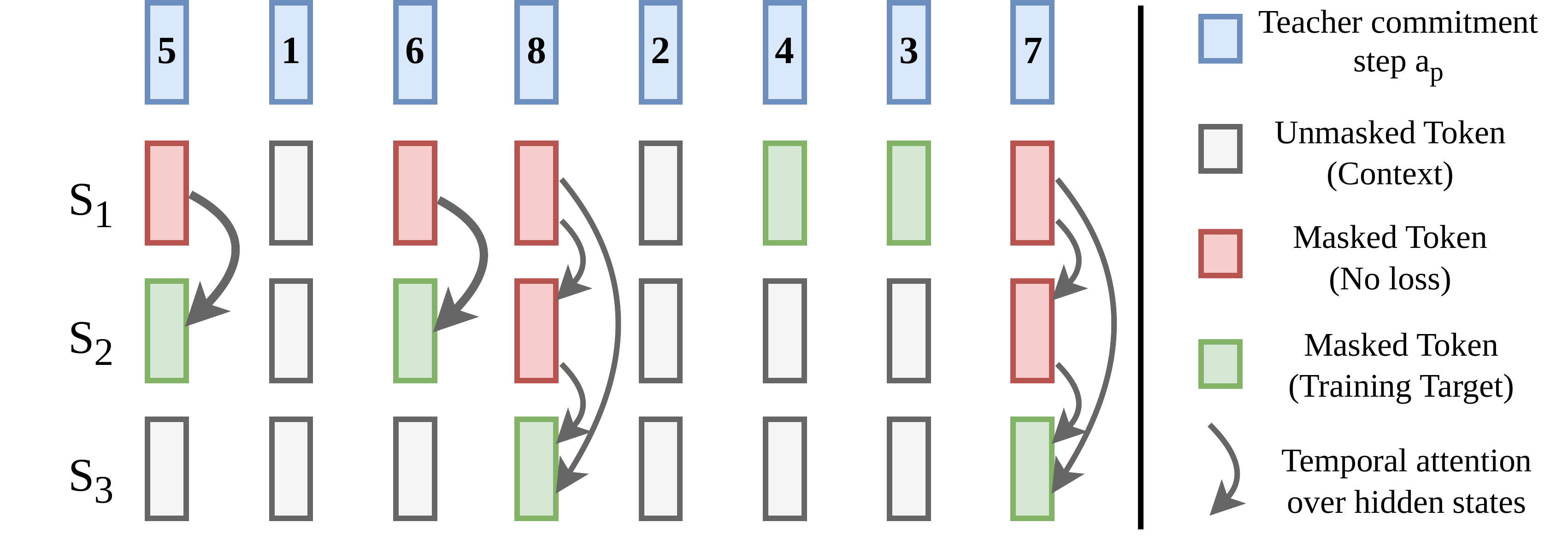}
\caption{\textbf{Sampled inputs and their training targets.} The top row gives the teacher's reveal step for each token position. The input rows show $S_1,S_2,S_3$ at $s_1=2,s_2=4,s_3=6$. Gray marks revealed context. Green marks tokens within the supervision window, and red marks masked positions beyond it. Arrows show reuse of earlier hidden representations at the same position and layer.}
\label{fig:training-stage}
\end{figure}

\paragraph{Parallel evaluation along a fixed trajectory.}
Spatial attention processes each reconstructed input independently. The training mask $A_{p,i,j}=\ind[j<i]m_{j,p}$ selects all earlier sampled inputs in which $p$ was masked. Keys and values precede the layer's temporal correction, so the sampled inputs can be evaluated together layer by layer (Appendix~\ref{app:properties}).

\subsection{Threshold commitment with temporal history}

For the current input $S_i$, let $\widehat y_p$ be the greedy prediction at position $p$ and $\gamma_p=\pi_{i,p}(\widehat y_p)$ its confidence. Let $\mathcal U_b$ contain the masked positions in the active block with $\widehat y_p\neq\mask$. The decoder commits the above-threshold predictions, with a top-1 fallback:
\begin{equation}
 C=
 \begin{cases}
 \{p\in\mathcal U_b:\gamma_p>\tau\},&\text{if nonempty},\\
 \{\arg\max_{p\in\mathcal U_b}\gamma_p\},&\text{otherwise}.
 \end{cases}
 \label{eq:commit}
\end{equation}
TDLM and our Fast-dLLM policy baseline share this commitment rule. Original DLM reveals one token per evaluation. Decoding settings are specified together in Section~\ref{sec:setup}.

For $\tau>1/2$, any token above threshold is the unique greedy prediction. Thus for the teacher token $y_p$ at each masked position $p$ in the active block,
\begin{equation}
 -\log\pi_{i,p}(y_p)<-\log\tau
 \quad\Longleftrightarrow\quad
 \widehat y_p=y_p \ \text{and}\ \gamma_p>\tau .
 \label{eq:teacher-commitment}
\end{equation}
Counting positions below this cross-entropy cutoff gives the number of threshold commitments that match the teacher in one evaluation.

Use the same teacher trajectories and layer-$\ell$ representations for all history budgets. Draw a nonterminal state from a fixed distribution and a masked position uniformly from its active block. Let $Y$ be the final teacher token there. A readout $q$ from a fixed family predicts $Y$ from $X_k$, the current representation and up to $k$ of the most recent retained entries at that position.

Let $A$ be the event that the greedy prediction $\widehat Y$ is a non-mask token with $q(\widehat Y\mid X_k)>\tau$. The \emph{threshold coverage} $\Pr(A)$ is the expected fraction of masked positions selected in one evaluation. For a tolerated error rate $\epsilon$ among these commitments, define
\[
 C_k^\star(\epsilon)
 = \sup_{q,\tau}\bigl\{\Pr(A) \,\big|\, \Pr(A,\widehat Y\ne Y)\le\epsilon\Pr(A)\bigr\}.
\]
For example, $C_k^\star(0.05)$ asks what fraction of masked positions can be committed while at least $95\%$ of the selected tokens match the teacher.

\begin{proposition}[History budget and threshold coverage]
\label{prop:history-coverage}
Write $X_{k+1}=(X_k,h)$, where $h$ is the next older entry or an empty entry for shorter histories. If the readout family can mask $h$, then for every $\epsilon\in[0,1]$,
\[
 C_{k+1}^\star(\epsilon)\ge C_k^\star(\epsilon), \qquad
 L_{k+1}^\star\le L_k^\star,
\]
where $L_k^\star=\inf_q\E[-\log q(Y\mid X_k)]$ is the optimal expected token cross-entropy.
\end{proposition}
\begin{proof}
At budget $k+1$, masking $h$ reproduces any $k$-entry readout's token probabilities at every position. At the same threshold, it commits the same tokens and makes the same errors. Every feasible coverage at $k$ is therefore attainable at $k+1$, and the masked readout also attains the same loss.
\end{proof}

On matched LLaDA teacher states, reading only the last retained evaluation per temporal layer achieves 67.1\% of the token CE reduction and 69\% of the coverage gain at 95\% teacher agreement obtained by reading all retained evaluations (Table~\ref{tab:gsm8k-teacher-history}).

The optimal loss reduction is the conditional usable information in $h$ about $Y$ for this readout family~\citep{hewitt2021conditional}. Let $a_k=L_k^\star-H(Y\mid X_k)$ be its approximation error above the Bayes-optimal loss. Then
\[
 L_k^\star-L_{k+1}^\star=I(Y;h\mid X_k)+a_k-a_{k+1}.
\]
When $X_k$ determines $h$, any loss reduction comes from giving the readout direct access to the cached computation.

Each decoding step reads the earlier temporal entries and appends the current keys and values for positions with $g_{i,p}=1$. Commitment releases a token's history. The lookahead block carries its accumulated history forward when it becomes active.

With $L_t$ temporal layers and total temporal width $d_t=Md_h$, TDLM adds $L_t(4dd_t+1)$ parameters. For block size $B$, at most $2B$ positions have temporal history at a decoding step. With up to $H$ retained evaluations per position, temporal retrieval costs $O(BHd_t)$ per layer per step apart from the projections, and the retained temporal KV occupies $O(L_tBHd_t)$ space.

\section{Experiments}
\label{sec:experiments}

We evaluate the quality--latency trade-off of TDLM on two backbones and study the contribution of its temporal history.

\subsection{Experimental setup}
\label{sec:setup}

\paragraph{Models and training.}
We augment frozen LLaDA-8B-Instruct and Dream-v0-Instruct-7B backbones. LLaDA uses sixteen dispersed temporal layers with four heads of width 64. Dream uses fourteen middle layers with four heads of width 84. For the fixed-setting comparison, we train the temporal modules for 750 updates on 3,500 self-generated trajectories per backbone. Appendix~\ref{app:training} details the training setup. The main MetaState baseline uses reference responses from the Tulu-3 SFT mixture~\citep{xia2026metastate}.

\paragraph{Tasks and scoring.}
We evaluate mathematical reasoning on MATH-500 and GSM8K, and code generation on MBPP sanitized~\citep{math500,cobbe2021gsm8k,austin2021mbpp,mbppsanitized}. We use the complete test sets of 500, 1,319, and 257 problems, respectively. Math-Verify checks mathematical answers, and dataset tests grade code solutions. Each setting uses one generation per problem. We report accuracy or pass@1, counting unparseable answers as incorrect.

\paragraph{Decoding and measurement.}
The fixed-setting comparison uses a 256-token generation budget and block size 32 (Appendix~\ref{app:fixed-setting-performance}). Confidence-based decoding uses $\tau=0.90$, following released evaluation settings~\citep{wu2025fastdllm,kim2025cdlm}. d3LLM uses its default entropy threshold $H=0.5$~\citep{qian2026d3llm}. On Dream, TDLM uses $\tau=0.90$ and d3LLM uses $H=0.4$. Original DLM commits one token per evaluation. We measure mean denoiser evaluations (NFE) and generation time per problem with batch size one on GH200 GPUs, excluding compilation and warmup. Speedup is measured against greedy original decoding on the same backbone and task. Appendix~\ref{app:evaluation} details scoring and timing.

TDLM uses one-block lookahead for temporal history in all decoding experiments. With spatial caching off in the fixed-setting comparison, the backbone recomputes its spatial attention keys and values (KV) at each step.

\subsection{Fixed-setting generation efficiency}
\label{sec:main-results}

Table~\ref{tab:main-results} compares TDLM, Fast-dLLM, and MetaState at a 256-token generation budget and block size 32. Appendix~\ref{app:performance-generalization} reports the full comparison (Table~\ref{tab:results256}) and evaluates LLaDA across block sizes and generation budgets.

\begin{table}[!htbp]
\centering
\newcommand{\fixedsettingextra}[1]{}
\newcommand{\fixedsettingmedal}[2]{#2}
\small
\setlength{\tabcolsep}{2.7pt}
\renewcommand{\arraystretch}{1.08}
\begin{tabular}{@{}llrrrrrrrr@{}}
\toprule
\multicolumn{2}{c}{} & \multicolumn{4}{c}{LLaDA-8B-Instruct} & \multicolumn{4}{c}{Dream-v0-Instruct-7B} \\
\cmidrule(lr){3-6}\cmidrule(lr){7-10}
Dataset & Model & ACC (\%) $\uparrow$ & NFE & Sec. $\downarrow$ & Speedup & ACC (\%) $\uparrow$ & NFE & Sec. $\downarrow$ & Speedup \\
\midrule
MATH-500 & TDLM & \fixedsettingmedal{medalbronze}{37.20} & 54.79 & \fixedsettingmedal{medalsilver}{0.829} & 4.26$\times$ & \fixedsettingmedal{medalsilver}{44.20} & 74.62 & \fixedsettingmedal{medalsilver}{1.194} & 3.05$\times$ \\
 & Fast-dLLM & \fixedsettingmedal{medalgold}{38.60} & 107.20 & \fixedsettingmedal{medalbronze}{1.535} & 2.30$\times$ & \fixedsettingmedal{medalbronze}{44.00} & 113.50 & \fixedsettingmedal{white}{1.646} & 2.21$\times$ \\
 & MetaState & \fixedsettingmedal{white}{36.80} & 107.56 & \fixedsettingmedal{white}{1.682} & 2.10$\times$ & \fixedsettingmedal{white}{40.80} & 115.88 & \fixedsettingmedal{white}{1.671} & 2.18$\times$ \\
\fixedsettingextra{%
 & CDLM & \fixedsettingmedal{white}{30.40} & 87.89 & \fixedsettingmedal{white}{1.584} & 2.23$\times$ & \fixedsettingmedal{white}{26.00} & 96.93 & \fixedsettingmedal{medalbronze}{1.438} & 2.53$\times$ \\
 & d3LLM & \fixedsettingmedal{white}{31.80} & 35.67 & \fixedsettingmedal{medalgold}{0.691} & 5.11$\times$ & \fixedsettingmedal{white}{42.40} & 53.76 & \fixedsettingmedal{medalgold}{0.820} & 4.45$\times$ \\
 & Original DLM & \fixedsettingmedal{medalsilver}{38.40} & 256.00 & \fixedsettingmedal{white}{3.535} & 1.00$\times$ & \fixedsettingmedal{medalgold}{44.60} & 256.00 & \fixedsettingmedal{white}{3.645} & 1.00$\times$ \\
}
\midrule
GSM8K & TDLM & \fixedsettingmedal{medalsilver}{79.08} & 40.08 & \fixedsettingmedal{medalsilver}{0.585} & 5.91$\times$ & \fixedsettingmedal{medalgold}{79.15} & 61.90 & \fixedsettingmedal{medalsilver}{0.931} & 3.71$\times$ \\
 & Fast-dLLM & \fixedsettingmedal{medalbronze}{78.85} & 76.72 & \fixedsettingmedal{medalbronze}{1.062} & 3.26$\times$ & \fixedsettingmedal{medalsilver}{79.00} & 97.35 & \fixedsettingmedal{white}{1.329} & 2.60$\times$ \\
 & MetaState & \fixedsettingmedal{medalgold}{81.12} & 77.36 & \fixedsettingmedal{white}{1.163} & 2.97$\times$ & \fixedsettingmedal{white}{75.21} & 101.79 & \fixedsettingmedal{white}{1.430} & 2.41$\times$ \\
\fixedsettingextra{%
 & CDLM & \fixedsettingmedal{white}{65.66} & 68.66 & \fixedsettingmedal{white}{1.527} & 2.26$\times$ & \fixedsettingmedal{white}{63.15} & 88.47 & \fixedsettingmedal{medalbronze}{1.236} & 2.79$\times$ \\
 & d3LLM & \fixedsettingmedal{white}{75.44} & 23.35 & \fixedsettingmedal{medalgold}{0.447} & 7.74$\times$ & \fixedsettingmedal{white}{77.33} & 40.54 & \fixedsettingmedal{medalgold}{0.587} & 5.88$\times$ \\
 & Original DLM & \fixedsettingmedal{white}{78.62} & 256.00 & \fixedsettingmedal{white}{3.457} & 1.00$\times$ & \fixedsettingmedal{medalbronze}{78.85} & 256.00 & \fixedsettingmedal{white}{3.452} & 1.00$\times$ \\
}
\midrule
MBPP & TDLM & \fixedsettingmedal{medalsilver}{46.69} & 36.98 & \fixedsettingmedal{medalgold}{0.517} & 6.27$\times$ & \fixedsettingmedal{medalsilver}{57.98} & 31.14 & \fixedsettingmedal{medalbronze}{0.467} & 7.15$\times$ \\
 & Fast-dLLM & \fixedsettingmedal{white}{43.97} & 52.13 & \fixedsettingmedal{medalbronze}{0.676} & 4.79$\times$ & \fixedsettingmedal{medalgold}{58.37} & 40.51 & \fixedsettingmedal{white}{0.559} & 5.97$\times$ \\
 & MetaState & \fixedsettingmedal{medalgold}{48.64} & 63.91 & \fixedsettingmedal{white}{0.917} & 3.53$\times$ & \fixedsettingmedal{medalgold}{58.37} & 59.88 & \fixedsettingmedal{white}{0.837} & 3.99$\times$ \\
\fixedsettingextra{%
 & CDLM & \fixedsettingmedal{medalbronze}{45.91} & 32.02 & \fixedsettingmedal{white}{0.820} & 3.95$\times$ & \fixedsettingmedal{white}{50.58} & 29.46 & \fixedsettingmedal{medalsilver}{0.401} & 8.31$\times$ \\
 & d3LLM & \fixedsettingmedal{white}{45.53} & 25.00 & \fixedsettingmedal{medalsilver}{0.526} & 6.16$\times$ & \fixedsettingmedal{medalbronze}{55.25} & 24.95 & \fixedsettingmedal{medalgold}{0.351} & 9.50$\times$ \\
 & Original DLM & \fixedsettingmedal{white}{44.75} & 256.00 & \fixedsettingmedal{white}{3.240} & 1.00$\times$ & \fixedsettingmedal{medalgold}{58.37} & 256.00 & \fixedsettingmedal{white}{3.334} & 1.00$\times$ \\
}
\bottomrule
\end{tabular}

\caption{\textbf{Fixed-setting comparison of TDLM, Fast-dLLM, and MetaState.} ACC is mathematical accuracy or code pass@1. NFE and seconds are per-problem means. Speedup is relative to greedy original decoding on the same backbone and task (Table~\ref{tab:results256}).}
\label{tab:main-results}
\end{table}

TDLM uses fewer denoiser evaluations and takes less time than Fast-dLLM and MetaState on all six backbone--task pairs. On Dream GSM8K, TDLM reaches 79.15\% accuracy in 0.931\,s per problem, compared with 79.00\% in 1.329\,s for Fast-dLLM. On LLaDA GSM8K, TDLM reaches 79.08\% in 0.585\,s, and MetaState reaches 81.12\% in 1.163\,s. The threshold sweeps in Figure~\ref{fig:gsm8k-time-accuracy} show TDLM on the GSM8K accuracy--latency frontier on both backbones.

\subsection{Spatial KV caching on GSM8K}
\label{sec:gsm8k-spatial-kv}

Spatial KV caching is orthogonal to TDLM's temporal retrieval. Since cache support and implementations vary across methods, we recompute spatial KV for all methods in the main comparison and evaluate caching separately on GSM8K. Prefix and Dual refresh spatial KV for the full sequence at the start of each block. Between refreshes, Prefix recomputes from the active block to the end of the sequence, and Dual recomputes the active and lookahead blocks. Both reuse spatial KV elsewhere. On Dream, both also recompute the position immediately before the active block to align predictions. Table~\ref{tab:gsm8k-spatial-kv} reports greedy GSM8K results for Off, Prefix, and Dual at the fixed settings above, including the released Fast-dLLM Dual decoder~\citep{wu2025fastdllm}.

\begin{table}[htbp]
\centering
\small
\setlength{\tabcolsep}{4pt}
\begin{tabular}{@{}lrrr@{\hspace{12pt}}rrr@{}}
\toprule
 & \multicolumn{3}{c}{LLaDA-8B-Instruct} & \multicolumn{3}{c}{Dream-v0-Instruct-7B} \\
\cmidrule(lr){2-4}\cmidrule(lr){5-7}
Method and cache mode & Acc. (\%) & NFE & Sec./q. & Acc. (\%) & NFE & Sec./q. \\
\midrule
Fast-dLLM (Off) & 78.85 & 76.72 & 1.062 & 79.00 & 97.35 & 1.329 \\
Fast-dLLM (Dual) & 79.53 & 82.36 & 0.748 & 79.45 & 106.22 & 0.985 \\
\midrule
TDLM (Off) & 79.08 & \textbf{40.08} & 0.585 & 79.15 & \textbf{61.90} & 0.931 \\
TDLM (Prefix) & 78.24 & 41.29 & 0.457 & 79.38 & 64.53 & 0.719 \\
TDLM (Dual) & \textbf{79.61} & 41.87 & \textbf{0.386} & \textbf{80.14} & 65.96 & \textbf{0.611} \\
\bottomrule
\end{tabular}
\caption{\textbf{Spatial KV caching on GSM8K.} NFE and generation time are means per question. Each block refresh counts as one denoiser evaluation. We run the released Fast-dLLM Dual sampling loop on the same compiled backbone implementation as the other configurations. Bold marks the highest accuracy and the lowest NFE and latency for each backbone.}
\label{tab:gsm8k-spatial-kv}
\end{table}

With Dual enabled for both methods, TDLM runs 1.94$\times$ as fast as Fast-dLLM on LLaDA and 1.61$\times$ as fast on Dream.

\begin{figure}[!htbp]
\subsection{What does temporal history contribute?}
\label{sec:history}
\centering
\begin{minipage}[t]{0.485\linewidth}
\vspace{0pt}

\raggedright

\paragraph{History at a fixed checkpoint.}
On LLaDA GSM8K, using all retained evaluations yields 79.08\% accuracy and 40.08 NFE, compared with 77.63\% and 44.06 NFE using only the last evaluation. On Dream, the last five evaluations give the highest measured accuracy (79.76\%), and all retained evaluations give the shortest generation time. Appendix~\ref{app:gsm8k-temporal-context} reports the comparison and memory costs.

\paragraph{Parameter-matched LoRA controls.}
We compare temporal retrieval with two LoRA controls that adapt QKV or feed-forward projections and process each state independently. The three models share backbone layer placement and trajectory supervision, with 67.11M trainable parameters each.

\smallskip
At $\tau=0.90$, TDLM scores 3.19--4.33 percentage points higher than the controls (Figure~\ref{fig:gsm8k-lora-frontier}, top). In the sweep, TDLM reaches 79.08\% at 0.585\,s per question and FFN LoRA peaks at 77.41\% at 0.564\,s (Figure~\ref{fig:gsm8k-lora-frontier}, bottom). Appendix~\ref{app:gsm8k-architecture} gives adapter settings and further architecture ablations.
\end{minipage}\hfill
\begin{minipage}[t]{0.485\linewidth}
\vspace{0pt}
\centering
\small
\begin{tabular*}{0.92\linewidth}{@{\extracolsep{\fill}}lrrr@{}}
\toprule
Method & Acc. (\%) & NFE & Sec. \\
\midrule
Fast-dLLM & 78.85 & 76.72 & 1.062 \\
LoRA (QKV) & 74.75 & 36.95 & 0.498 \\
LoRA (FFN) & 75.89 & 37.91 & 0.508 \\
TDLM & \textbf{79.08} & 40.08 & 0.585 \\
\bottomrule
\end{tabular*}\par\smallskip
\includegraphics[width=\linewidth]{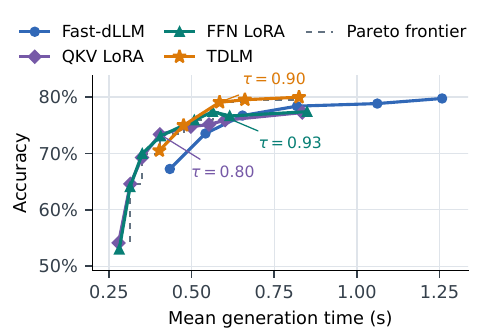}
\setlength{\abovecaptionskip}{2pt}
\caption{\textbf{LLaDA GSM8K LoRA controls.} Results at $\tau=0.90$ (top) and accuracy--latency curves with a dashed Pareto frontier (bottom).}
\label{fig:gsm8k-lora-frontier}
\end{minipage}
\end{figure}

\section{Related work}
\label{sec:related}

\paragraph{Parallel decoding and diffusion acceleration.}
Fast-dLLM~\citep{wu2025fastdllm} combines confidence-based parallel decoding with KV-cache reuse. dKV-Cache~\citep{ma2025dkvcache} uses delayed and conditioned caching to reuse key and value states during denoising. Training-based methods include CDLM~\citep{kim2025cdlm}, which combines consistency training with blockwise causal attention, and d3LLM~\citep{qian2026d3llm}, which uses pseudo-trajectory distillation and entropy-based multi-block decoding. TDLM studies position-aligned temporal retrieval as a way to make earlier denoising computations useful for later predictions.

\paragraph{Self-conditioning and persistent representations.}
Analog Bits~\citep{chen2023analogbits} conditions diffusion predictions on earlier model predictions. Recurrent Interface Networks~\citep{jabri2022rin} introduce latent self-conditioning across reverse diffusion steps for image and video generation. In language models, RCD~\citep{hu2026rcd} recycles discarded-token representations as residual context for subsequent denoising. MetaState~\citep{xia2026metastate} equips a frozen discrete diffusion model with fixed-size recurrent working memory, while ALoDLM~\citep{fang2026alodlmadaptivelyloopeddiffusion} uses token-adaptive latent recurrence. TDLM retrieves an explicit history of same-position representations at multiple layers of a frozen backbone and evaluates the sampled trajectory inputs in parallel during training.

\section{Discussion and future work}
\label{sec:limitations}

Teacher trajectories provide a direct training signal for predicting upcoming commitments. Varying the teacher and the scope of fine-tuning would test how stronger supervision and joint adaptation change the value of temporal retrieval. Evaluations on longer responses with output-aware stopping would extend the accuracy--latency comparison beyond fixed-budget generation.

\section{Conclusion}

TDLM reuses earlier same-position representations in a frozen backbone and maintains history for masked positions in the active block and one lookahead block during decoding. It reaches the GSM8K accuracy--latency frontier on LLaDA and Dream. On LLaDA GSM8K at a 256-token budget, combining temporal reuse with dual spatial KV caching yields 79.61\% accuracy, an estimated 8.96$\times$ speedup over a separately timed uncached greedy baseline, and a 1.94$\times$ speedup over Fast-dLLM with dual KV caching.

\bibliographystyle{tdlmplainnat}
\bibliography{references}
\clearpage
\appendix
\section{Training details}
\label{app:training}

The fixed-setting comparison uses temporal modules trained for 750 optimizer updates with the backbone frozen. The output projections and the per-layer sink logits $b_\ell$ start at zero.

\paragraph{Teacher trajectories and state sampling.}
Each frozen backbone generates its own trajectories with greedy top-1 decoding in 32-token blocks. Each denoiser evaluation commits the highest-confidence masked token in the active block, yielding $T=128$ or $T=256$ steps for the corresponding response budget. We evaluate $K=64$ or $K=128$ states, respectively, including the all-masked response and $K-1$ distinct interior steps sampled uniformly. The terminal state $s_K=T$ serves as the last supervision boundary.

\paragraph{Optimization and cost.}
We use AdamW ($\beta_1=0.9$, $\beta_2=0.999$, $\epsilon=10^{-8}$) with zero weight decay and a gradient norm cap of $1.0$. The learning rate increases linearly to $10^{-4}$ over 40 updates, then follows a cosine schedule with a 1,000-update horizon and a floor of $10^{-5}$. Both runs use seed 42, gradient checkpointing, and BF16 autocasting with FP32 temporal parameters. Each update uses one trajectory length across four NVIDIA GH200 GPUs. The per-GPU batch sizes are 2 for $T=128$ and 1 for $T=256$, giving global batch sizes of 8 and 4 with no gradient accumulation. The target sampling mix is one 128-token trajectory for every two 256-token trajectories. The 750-update runs take 5.24 GPU-hours for LLaDA and 4.49 for Dream, measured as four times the logged training duration. This duration includes compilation, validation, and checkpointing and excludes teacher trajectory generation.

\paragraph{LLaDA.}
Temporal modules are installed at zero-based layer indices $\{1,3,5,\ldots,31\}$. Each uses four heads of width 64, adding 67,108,880 parameters in total, approximately 0.84\% of the nominal 8B backbone.

\paragraph{Dream.}
Temporal modules occupy the middle fourteen layers, at zero-based indices 7--20, with four heads of width 84 and the same parameter count as LLaDA. The $p-1$ prediction alignment uses the final prompt representation for the first response token. Temporal memory and commitment masks follow the target response positions.

\paragraph{Trainable parameters and data scale.}
Table~\ref{tab:training-scale} compares the added trainable parameters and training-data volume of the LLaDA-based methods. Counts use each method's original unit of examples, prompts, or trajectories.

\begin{table}[htbp]
\centering
\small
\setlength{\tabcolsep}{6pt}
\renewcommand{\arraystretch}{1.15}
\begin{tabular}{@{}lrp{0.51\linewidth}@{}}
\toprule
Method & \shortstack{Added trainable\\parameters} & Training data \\
\midrule
d3LLM~\citep{qian2026d3llm} & 335.54M (LoRA) & $\approx$92k examples \\
CDLM~\citep{kim2025cdlm} & 176.13M (LoRA) & $\approx$15k prompts; trajectories collected at two temperatures per prompt \\
MetaState~\citep{xia2026metastate} & 51.08M & 15,543 instruction--response examples \\
\tdlm{} & 67.11M & 3,500 self-generated trajectories \\
\bottomrule
\end{tabular}
\caption{Added trainable parameters and training-data scale for LLaDA-based models. M denotes one million parameters. MetaState uses reference responses from the Tulu-3 SFT mixture~\citep{xia2026metastate}, and \tdlm{} generates its training trajectories from SmolTalk prompts~\citep{smoltalk}.}
\label{tab:training-scale}
\end{table}

\section{Properties of temporal evaluation}
\label{app:properties}

\subsection{Temporal attention with a learned sink}
\label{app:temporal-details}

For a fixed layer $\ell$ and target position $p$, let $q_{i,a}$, $k_{i,a}$, and $v_{i,a}$ be the bias-free temporal projections of $h_{\ell,p}(S_i)$ for head $a$. Queries and keys use RoPE with angles $r_i\,10000^{-2u/d_h}$ for $u=0,\ldots,d_h/2-1$. Relative position enters the score $e_a^{i,j}=q_{i,a}^{\top}k_{j,a}/\sqrt{d_h}$ through the revealed-token difference $r_j-r_i$.

The shared learned sink logit $b_\ell$ adds a zero-value entry to each head's attention. For retained indices $\mathcal J_p(i)$, the head output and its normalizer are
\begin{equation}
 z_{i,a}=\frac{\sum_{j\in\mathcal J_p(i)}\exp(e_a^{i,j})v_{j,a}}{Z_{i,a}},
 \qquad
 Z_{i,a}=\exp(b_\ell)+\sum_{j\in\mathcal J_p(i)}\exp(e_a^{i,j}).
 \label{eq:retrieval}
\end{equation}
The bias-free output projection $W^O_\ell\in\R^{d\times d_t}$ maps the concatenated head outputs to the residual width in Eq.~\eqref{eq:temporal}.

\subsection{Sink, cold start, and initialization}

The sink takes probability mass $\pi_{\varnothing,a}=\exp(b_\ell)/Z_{i,a}$ in head $a$, leaving $1-\pi_{\varnothing,a}$ for the history. If $\|v_{j,a}\|_2\leq V$ for every retained value vector, the triangle inequality bounds the residual update $\Delta h_{\ell,p}(S_i)$:
\begin{equation}
 \|\Delta h_{\ell,p}(S_i)\|_2
 \leq g_{i,p}\|W^O_\ell\|_2 V
 \sqrt{\sum_{a=1}^{M}(1-\pi_{\varnothing,a})^2}.
 \label{eq:sink-bound}
\end{equation}
\begin{proposition}[Cold start and initialization]
For history-only temporal attention, the first forward pass of each request has zero temporal correction at every layer. If each $W^O_\ell$ is initialized to zero, the entire initialized TDLM denoiser equals the backbone on every input along any fixed trajectory, in exact arithmetic.
\end{proposition}
\begin{proof}
On the first pass, $\mathcal J_p(0)$ is empty, so the only attention entry is the zero-value sink. At initialization, multiplication by $W^O_\ell=0$ removes the correction even when history is available. Induction over backbone layers gives identical representations and output distributions.
\end{proof}

\subsection{Parallel and sequential evaluation}
\begin{proposition}[Parallel evaluation on a fixed trajectory]
Parallel and sequential evaluation of the same sampled inputs agree in exact arithmetic when both use identical spatial attention and temporal masks, with dropout disabled.
\end{proposition}
\begin{proof}
The input embeddings agree for each $S_i$. Suppose the representations agree at the input to a layer. The same spatial attention gives identical $h_{\ell,p}(S_i)$, from which each evaluation computes keys and values before the temporal correction. The parallel mask selects the same entries that sequential evaluation caches. The temporal update and feed-forward sublayer therefore give the same outputs. Induction over layers establishes the claim.
\end{proof}
Gradients propagate through the frozen backbone computations to train the temporal modules.

\section{Evaluation details}
\label{app:evaluation}

\paragraph{Answer scoring.}
MATH-500 and GSM8K prompts request a solution ending in a boxed answer. The last boxed answer is checked for mathematical equivalence using Math-Verify 0.8.0. MBPP prompts request complete Python code with a signature derived from the reference implementation, without exposing the function body or tests. Extracted code is checked against the dataset tests with a three-second timeout and a 512 MB memory limit. Failed or unparseable answers count as incorrect. Scoring uses the text before the first EOS/EOT marker.

\paragraph{Runtime measurement.}
Timed generation covers the full token budget and excludes compilation, warmup, tokenization, external text decoding, and grading. It includes native decoder-loop work where present. TDLM uses a fused temporal kernel and BF16 temporal keys and values. The fixed-setting runs use PyTorch 2.9.0+cu128 and compiled generation with eager fallback for long prompts.

At the fixed settings, TDLM, Fast-dLLM and original LLaDA each use eager fallback for four long MATH-500 prompts. MetaState compiles the full LLaDA test set. The Dream runs use three long-prompt fallbacks per method. Measured latency includes these fallbacks.

\section{Architecture ablations on GSM8K}
\label{app:gsm8k-architecture}

We compare temporal-layer allocations on LLaDA-8B-Instruct and Dream-v0-Instruct-7B and vary temporal width on LLaDA. The LLaDA runs use the fixed GSM8K setting from Section~\ref{sec:setup} and measure full-budget generation in separate GH200 jobs. NFE and generation time are means over all 1,319 problems. Timing excludes compilation and warmup.

\subsection{Different temporal-layer allocations}

\paragraph{LLaDA.}
The sixteen dispersed layers used in the main results reach 79.08\% accuracy, a 1.22-point gain over eight dispersed layers at similar latency. At the same layer count, the dispersed allocation delivers 79.08\% in 0.585\,s and the middle-layer allocation delivers 79.00\% in 0.683\,s (Table~\ref{tab:gsm8k-layer-allocation}). Training with temporal modules in all 32 backbone layers ran out of GPU memory.

\begin{table}[htbp]
\centering
\small
\begin{tabular}{lrrr}
\toprule
Temporal layers & Accuracy (\%) & NFE & Avg. sec \\
\midrule
8, dispersed & 77.86 & 40.15 & 0.571 \\
16, middle & 79.00 & 45.72 & 0.683 \\
\textbf{16, dispersed (main)} & 79.08 & 40.08 & 0.585 \\
32, all & \multicolumn{3}{c}{Training OOM} \\
\bottomrule
\end{tabular}
\caption{\textbf{LLaDA temporal-layer allocations on GSM8K.} Every module uses four heads of width 64. The three evaluated models use the same 750-update training recipe.}
\label{tab:gsm8k-layer-allocation}
\end{table}

\paragraph{Dream.}
The allocation experiment uses four heads of width 84 and 1,976 training updates for each layout. We evaluate 256-token greedy decoding with block size 32 and $\tau=0.90$, pairing each layout with the middle-layer reference on the same problems. The middle fourteen layers achieve the highest accuracy (80.67\%), 1.06 percentage points above the dispersed allocation at the same layer count (Table~\ref{tab:gsm8k-dream-layer-allocation}). Switching from dispersed to middle placement at this depth changes 77 answers from incorrect to correct and 63 from correct to incorrect.

\begin{table}[htbp]
\centering
\small
\begin{tabular}{lrrr}
\toprule
Temporal layers & Accuracy (\%) & NFE & Avg. sec \\
\midrule
7, dispersed & 78.70 & 59.67 & 0.870 \\
14, middle & \textbf{80.67} & 61.69 & 0.939 \\
14, dispersed & 79.61 & 59.38 & 0.905 \\
28, all & 77.03 & 58.72 & 0.966 \\
\bottomrule
\end{tabular}
\caption{\textbf{Dream temporal-layer allocations on GSM8K.}}
\label{tab:gsm8k-dream-layer-allocation}
\end{table}

\subsection{Temporal width and head configurations}

At a total width of 256 and 750 updates, four heads of width 64 reach 79.00\% accuracy, compared with 78.47\% for two heads of width 128. The narrower configurations use 500 updates (Table~\ref{tab:gsm8k-temporal-width}).

\begin{table}[htbp]
\centering
\small
\begin{tabular}{lrrrrr}
\toprule
Heads $\times$ head width & Total width & Updates & Accuracy (\%) & NFE & Avg. sec \\
\midrule
$1 \times 64$ & 64 & 500 & 76.35 & 47.34 & 0.682 \\
$1 \times 128$ & 128 & 500 & 77.71 & 47.66 & 0.701 \\
$2 \times 128$ & 256 & 750 & 78.47 & 45.84 & 0.688 \\
$4 \times 64$ & 256 & 750 & 79.00 & 45.72 & 0.683 \\
\bottomrule
\end{tabular}
\caption{Temporal-width and head configurations on GSM8K. All models augment the middle sixteen backbone layers (zero-based indices 8--23). Total width is the product of head count and head width.}
\label{tab:gsm8k-temporal-width}
\end{table}

\subsection{LoRA control configuration}

The LoRA controls in Figure~\ref{fig:gsm8k-lora-frontier} use ranks $(171,171,170)$ for the query, key, and value projections or $(86,85,85)$ for the gate, up, and down projections. Each adds 67,108,864 trainable parameters, compared with 67,108,880 in TDLM. Both train for 750 updates and merge the learned adapters into the backbone for inference. The LoRA sweeps cover eight confidence thresholds from 0.5 to 0.99 on the same 1,319 GSM8K problems as the fixed-threshold comparison.

\section{Temporal context during inference on GSM8K}
\label{app:gsm8k-temporal-context}

At fixed checkpoints and decoding settings, we compare current-only attention, the last one, three, or five retained evaluations, and all retained evaluations within the one-block lookahead policy. Current-only attention retains the learned query, key, and sink gate.

Finite histories use circular KV slots per position. Retaining all evaluations uses eight-step pages returned to the allocator at commitment. A fixed device page table permits CUDA Graph reuse. Current-only attention uses the original temporal implementation. GPU memory measures peak allocation after warmup, including weights and workspaces.

\begin{table}[htbp]
\centering
\small
\setlength{\tabcolsep}{3.5pt}
\begin{tabular}{@{}lrrrr@{\hspace{8pt}}rrrr@{}}
\toprule
 & \multicolumn{4}{c}{LLaDA-8B-Instruct} & \multicolumn{4}{c}{Dream-v0-Instruct-7B} \\
\cmidrule(lr){2-5}\cmidrule(lr){6-9}
Temporal history & \shortstack[r]{Acc.} & NFE & Sec. & \shortstack[r]{Peak Overhead \\ (GiB, \%)} & \shortstack[r]{Acc.} & NFE & Sec. & \shortstack[r]{Peak Overhead \\ (GiB, \%)} \\
\midrule
Current only (gated) & 73.46 & 42.64 & 0.623 & 0.000 (0.00\%) & 78.01 & 76.59 & 1.144 & 0.000 (0.00\%) \\
Last 1 & 77.63 & 44.06 & 0.637 & +0.004 (+0.03\%) & 79.23 & 72.23 & 1.066 & +0.005 (+0.03\%) \\
Last 3 & 78.01 & 42.36 & 0.631 & +0.012 (+0.08\%) & 79.30 & 66.61 & 0.993 & +0.014 (+0.09\%) \\
Last 5 & 78.70 & 41.57 & 0.601 & +0.020 (+0.13\%) & 79.76 & 66.57 & 0.962 & +0.022 (+0.16\%) \\
All retained evaluations & 79.08 & 40.08 & 0.585 & +0.146 (+0.96\%) & 79.15 & 61.90 & 0.931 & +0.173  (+1.19\%) \\
\bottomrule
\end{tabular}
\caption{\textbf{Temporal history length during GSM8K decoding.} Each backbone uses a fixed trained checkpoint. Current-only retains the Q/K/sink gate. GPU memory includes weights and compiled workspaces. Overhead is relative to each backbone's current-only peak.}
\label{tab:gsm8k-temporal-context}
\end{table}

Retaining all evaluations adds 0.146\,GiB on LLaDA and 0.173\,GiB on Dream relative to current-only attention. Retaining the last five adds 0.020\,GiB and 0.022\,GiB, respectively (Table~\ref{tab:gsm8k-temporal-context}).

\paragraph{Effect of history length on matched teacher states.}
We replay the first three 32-token blocks of each 256-token LLaDA teacher trajectory on 32 held-out GSM8K problems. At every step, all four conditions receive the same partially revealed input and use the same TDLM checkpoint to predict the final teacher token at every masked position in the active block. The teacher commits one token to determine the next input for all four conditions. We use one-block lookahead and vary how many retained evaluations each temporal layer reads at the scored position.

We average token cross-entropy over the masked positions and sweep the confidence threshold separately for each history length on these replayed states. Coverage is the fraction of masked positions selected by the threshold branch of Eq.~\eqref{eq:commit}. We report the highest coverage at which at least 95\% of the selected tokens match the teacher (Table~\ref{tab:gsm8k-teacher-history}).

\begin{table}[htbp]
\centering
\small
\begin{tabular}{@{}lrrr@{}}
\toprule
Temporal history & Token CE $\downarrow$ &
  \shortstack[r]{Coverage at 95\%\\agreement $\uparrow$} &
  \shortstack[r]{Tokens committed\\per state $\uparrow$} \\
\midrule
No history               & 1.180          & 48.25\%          & 7.96 \\
Last 1                   & 1.070          & 51.88\%          & 8.56 \\
Last 2                   & 1.027          & 52.12\%          & 8.60 \\
All retained evaluations & \textbf{1.016} & \textbf{53.52\%} & \textbf{8.83} \\
\bottomrule
\end{tabular}
\caption{\textbf{Temporal history length on matched GSM8K teacher states.} All four conditions share 3,072 LLaDA states and 50,688 masked-position observations.}
\label{tab:gsm8k-teacher-history}
\end{table}

\section{Counterfactual early commitments on GSM8K}
\label{app:gsm8k-commitment-interventions}

To measure how earlier commitments affect final answers, we add TDLM proposals to top-1 decoding and compare the resulting answers with a paired top-1 reference. The two backbones use the same 1,024 randomly sampled GSM8K problems and the decoding settings in Section~\ref{sec:setup}. TDLM fills its temporal cache by replaying the reference trajectory. From a random starting step, we find the first state before answer completion with an above-threshold TDLM proposal at an active-block position that the original decoder would leave masked. At that step, we keep the top-1 commitment and add the highest-confidence eligible proposal.

With one extra commitment, final-answer agreement reaches 98.19\% on 992 eligible LLaDA problems and 98.56\% on 904 eligible Dream problems. Correct-to-incorrect changes affect $2/786$ (0.25\%) and $1/695$ (0.14\%) reference-correct answers. Both matched-position controls reproduce these three changes, one using training-style history subsampling and the other committing the original denoiser's proposal. One of the three added tokens matches the terminal reference token at that position, the future-stability criterion used by TraceLock~\citep{sun2026tracelock}.

The one-, three-, and five-commitment branches use a top-1 reference generated in the same run and share the same prompt, starting step, and intervened trajectory up to their respective intervention limits. After the first intervention, we draw up to four progress targets uniformly without replacement from the remaining revealed-token counts before the reference answer is fully committed. TDLM follows the evolving trajectory and adds a proposal at the first eligible step at or beyond each target while the answer is unfinished. Each branch resumes top-1 decoding after reaching its intervention limit.

\begin{table}[htbp]
\centering
\caption{\textbf{One, three, and five extra commitments on GSM8K.} Each backbone uses a common cohort of prompts that reach five eligible interventions. Agreement measures mathematical equivalence with the top-1 reference. The two transition rates are relative to the reference-correct and reference-incorrect cohorts, respectively.}
\label{tab:gsm8k-commitment-interventions}
\medskip
\setlength{\tabcolsep}{5pt}
\begin{tabular}{lrrrr}
\toprule
Backbone & Interventions & Agreement (\%) & Correct $\to$ incorrect (\%) & Incorrect $\to$ correct (\%) \\
\midrule
LLaDA & 1 & 98.37 & 0.15 (1/680) & 0.00 (0/179) \\
 & 3 & 98.14 & 0.29 (2/680) & 0.00 (0/179) \\
 & 5 & 97.56 & 0.74 (5/680) & 0.56 (1/179) \\
\midrule
Dream  & 1 & 97.90 & 0.17 (1/585) & 0.00 (0/176) \\
 & 3 & 97.77 & 0.17 (1/585) & 0.00 (0/176) \\
 & 5 & 97.63 & 0.17 (1/585) & 0.57 (1/176) \\
\bottomrule
\end{tabular}
\end{table}

On the common five-event cohort for each backbone, final-answer agreement is about 97.6\% at five interventions (Table~\ref{tab:gsm8k-commitment-interventions}). Each backbone repairs one answer. LLaDA adds the required boxed format to a correct numerical answer, and Dream corrects an arithmetic error. The net accuracy changes are $-0.47$ percentage points on LLaDA and $0.00$ on Dream. Across all prompts with a first eligible event, the policy with up to five interventions repairs one answer per backbone and changes $5/781$ (0.64\%) reference-correct LLaDA answers and $2/697$ (0.29\%) reference-correct Dream answers to incorrect.

\section{Performance and generalization}
\label{app:performance-generalization}

\subsection{Performance}
\label{app:fixed-setting-performance}

Table~\ref{tab:results256} compares methods at the fixed settings in Section~\ref{sec:setup}. TDLM reduces NFE and latency relative to Fast-dLLM on all six backbone--task pairs. It achieves a 7.15$\times$ speedup over greedy original decoding on Dream MBPP. Figure~\ref{fig:math500-mbpp-frontier} shows the quality--latency trade-off on MATH-500 and MBPP across decoding settings.

\begin{table}[!htbp]
\centering
\setlength{\fboxsep}{1pt}
\newcommand{\fixedsettingextra}[1]{#1}
\newcommand{\fixedsettingmedal}[2]{\colorbox{#1}{#2}}

\caption{\textbf{Fixed-setting performance on two backbones.} ACC is mathematical accuracy or code pass@1. NFE and seconds are per-problem means. Speedup is relative to greedy original decoding on the same backbone and task. \colorbox{medalgold}{Gold}, \colorbox{medalsilver}{silver}, and \colorbox{medalbronze}{bronze} mark the three best distinct displayed values for accuracy and latency within each backbone--task group. Ties share a rank.}
\label{tab:results256}
\label{tab:dream-results256}
\end{table}

\begin{figure}[!htbp]
\centering
\includegraphics[width=\linewidth]{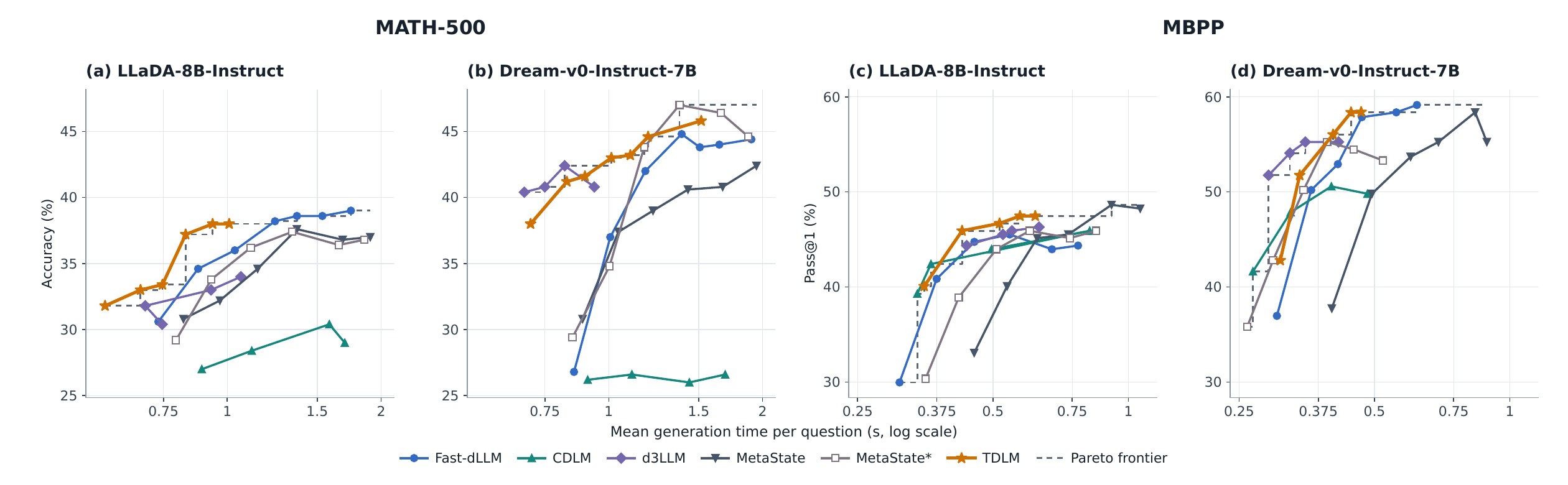}
\caption{\textbf{Quality versus generation time on MATH-500 and MBPP.} The panels plot MATH-500 accuracy and MBPP pass@1 against mean generation time per question. Dashed lines mark the Pareto frontier for each task and backbone.}
\label{fig:math500-mbpp-frontier}
\end{figure}
\FloatBarrier

\subsection{Generalization across block sizes}
\label{app:llada-block-size}

Table~\ref{tab:llada-block-sizes256} compares block sizes 16 and 64 at a fixed generation budget of 256 tokens, using the same LLaDA checkpoints and test sets as in Table~\ref{tab:results256}. Each method uses the same decoding thresholds at both sizes. Both generalization comparisons score mathematical answers from CDLM and d3LLM using direct Math-Verify extraction.

TDLM reaches 3.74--5.16$\times$ speedups over original decoding at block size 16 and 4.50--7.18$\times$ at block size 64. Increasing the block size reduces TDLM's mean latency by 18--29\% across the three tasks (Table~\ref{tab:llada-block-sizes256}).

\begin{table}[!htbp]
\centering
\small
\setlength{\tabcolsep}{2.7pt}
\setlength{\fboxsep}{1pt}
\renewcommand{\arraystretch}{1.08}
\begin{tabular}{@{}llrrrrrrrr@{}}
\toprule
\multicolumn{2}{c}{} & \multicolumn{4}{c}{Block 16} & \multicolumn{4}{c}{Block 64} \\
\cmidrule(lr){3-6}\cmidrule(lr){7-10}
Dataset & Model & ACC (\%) $\uparrow$ & NFE & Sec. $\downarrow$ & Speedup & ACC (\%) $\uparrow$ & NFE & Sec. $\downarrow$ & Speedup \\
\midrule
MATH-500 & TDLM & \colorbox{medalsilver}{37.00} & 61.72 & \colorbox{medalsilver}{0.946} & 3.74$\times$ & \colorbox{medalbronze}{37.60} & 50.20 & \colorbox{medalsilver}{0.776} & 4.50$\times$ \\
& Fast-dLLM & \colorbox{medalsilver}{37.00} & 111.93 & \colorbox{medalbronze}{1.594} & 2.22$\times$ & \colorbox{white}{36.80} & 104.89 & \colorbox{medalbronze}{1.482} & 2.36$\times$ \\
& MetaState & \colorbox{medalbronze}{36.60} & 112.02 & \colorbox{white}{1.725} & 2.05$\times$ & \colorbox{medalsilver}{38.00} & 105.60 & \colorbox{white}{1.604} & 2.18$\times$ \\
& CDLM & \colorbox{white}{25.40} & 97.77 & \colorbox{white}{2.276} & 1.55$\times$ & \colorbox{white}{27.80} & 81.32 & \colorbox{white}{1.565} & 2.23$\times$ \\
& d3LLM & \colorbox{white}{32.40} & 38.59 & \colorbox{medalgold}{0.744} & 4.75$\times$ & \colorbox{white}{29.80} & 36.85 & \colorbox{medalgold}{0.773} & 4.52$\times$ \\
& Original DLM & \colorbox{medalgold}{37.60} & 256.00 & \colorbox{white}{3.536} & 1.00$\times$ & \colorbox{medalgold}{38.60} & 256.00 & \colorbox{white}{3.493} & 1.00$\times$ \\
\midrule
GSM8K & TDLM & \colorbox{medalgold}{79.98} & 47.30 & \colorbox{medalsilver}{0.695} & 4.94$\times$ & \colorbox{medalsilver}{79.76} & 36.12 & \colorbox{medalsilver}{0.533} & 6.33$\times$ \\
& Fast-dLLM & \colorbox{medalsilver}{79.23} & 82.68 & \colorbox{medalbronze}{1.132} & 3.03$\times$ & \colorbox{medalbronze}{79.68} & 74.07 & \colorbox{medalbronze}{1.004} & 3.36$\times$ \\
& MetaState & \colorbox{white}{77.94} & 83.10 & \colorbox{white}{1.227} & 2.80$\times$ & \colorbox{medalgold}{79.83} & 74.50 & \colorbox{white}{1.087} & 3.10$\times$ \\
& CDLM & \colorbox{white}{60.05} & 77.26 & \colorbox{white}{1.586} & 2.16$\times$ & \colorbox{white}{60.12} & 62.26 & \colorbox{white}{1.367} & 2.47$\times$ \\
& d3LLM & \colorbox{white}{73.54} & 27.01 & \colorbox{medalgold}{0.529} & 6.49$\times$ & \colorbox{white}{71.49} & 23.45 & \colorbox{medalgold}{0.487} & 6.92$\times$ \\
& Original DLM & \colorbox{medalbronze}{78.47} & 256.00 & \colorbox{white}{3.431} & 1.00$\times$ & \colorbox{white}{77.79} & 256.00 & \colorbox{white}{3.372} & 1.00$\times$ \\
\midrule
MBPP & TDLM &  \colorbox{medalbronze}{45.91} & 45.12 & \colorbox{medalsilver}{0.632} & 5.16$\times$ & \colorbox{medalgold}{47.86} & 31.99 & \colorbox{medalsilver}{0.449} & 7.18$\times$ \\
& Fast-dLLM & \colorbox{white}{45.14} & 58.21 & \colorbox{medalbronze}{0.754} & 4.32$\times$ & \colorbox{medalbronze}{46.69} & 46.47 & \colorbox{medalbronze}{0.596} & 5.41$\times$ \\
& MetaState & \colorbox{medalgold}{46.69} & 70.01 & \colorbox{white}{0.986} & 3.31$\times$ & \colorbox{medalsilver}{47.47} & 60.34 & \colorbox{white}{0.839} & 3.85$\times$ \\
& CDLM & \colorbox{white}{43.97} & 42.63 & \colorbox{white}{0.837} & 3.89$\times$ & \colorbox{white}{42.80} & 33.85 & \colorbox{white}{0.751} & 4.30$\times$ \\
& d3LLM & \colorbox{white}{44.36} & 29.67 & \colorbox{medalgold}{0.559} & 5.83$\times$ & \colorbox{white}{43.19} & 17.61 & \colorbox{medalgold}{0.366} & 8.81$\times$ \\
& Original DLM & \colorbox{medalsilver}{46.30} & 256.00 & \colorbox{white}{3.260} & 1.00$\times$ & \colorbox{medalbronze}{46.69} & 256.00 & \colorbox{white}{3.226} & 1.00$\times$ \\
\bottomrule
\end{tabular}
\caption{\textbf{LLaDA block sizes with a 256-token generation budget.} ACC is mathematical accuracy or code pass@1. NFE and seconds are per-problem means. Speedup divides the displayed original DLM latency by the displayed method latency on the same task and block size. Medal colors follow Table~\ref{tab:results256}, which reports block size 32.}
\label{tab:llada-block-sizes256}
\end{table}

\subsection{Generalization across generation budgets}

With block size fixed at 32, TDLM delivers 3.50--4.84$\times$ speedups over original decoding at 128 tokens and 4.81--8.54$\times$ at 384 tokens. It uses fewer evaluations and less time than Fast-dLLM on all six task--budget pairs. At 384 tokens, it also improves MATH-500 accuracy by 1.20 percentage points over Fast-dLLM (Table~\ref{tab:llada-generation-budgets}).

\begin{table}[!htbp]
\centering
\small
\setlength{\tabcolsep}{2.7pt}
\setlength{\fboxsep}{1pt}
\renewcommand{\arraystretch}{1.08}
\begin{tabular}{@{}llrrrrrrrr@{}}
\toprule
\multicolumn{2}{c}{} & \multicolumn{4}{c}{128 tokens} & \multicolumn{4}{c}{384 tokens} \\
\cmidrule(lr){3-6}\cmidrule(lr){7-10}
Dataset & Model & ACC (\%) $\uparrow$ & NFE & Sec. $\downarrow$ & Speedup & ACC (\%) $\uparrow$ & NFE & Sec. $\downarrow$ & Speedup \\
\midrule
MATH-500 & TDLM & \colorbox{medalsilver}{36.00} & 32.60 & \colorbox{medalgold}{0.366} & 3.50$\times$ & \colorbox{medalsilver}{42.80} & 70.83 & \colorbox{medalsilver}{1.389} & 4.81$\times$ \\
& Fast-dLLM & \colorbox{medalgold}{36.40} & 64.28 & \colorbox{medalbronze}{0.670} & 1.91$\times$ & \colorbox{medalbronze}{41.60} & 134.01 & \colorbox{white}{2.443} & 2.73$\times$ \\
& MetaState & \colorbox{medalbronze}{35.80} & 65.14 & \colorbox{white}{0.742} & 1.72$\times$ & \colorbox{white}{41.20} & 135.23 & \colorbox{white}{2.599} & 2.57$\times$ \\
& CDLM & \colorbox{white}{30.80} & 46.36 & \colorbox{white}{1.003} & 1.28$\times$ & \colorbox{white}{31.00} & 91.48 & \colorbox{medalbronze}{1.983} & 3.37$\times$ \\
& d3LLM & \colorbox{white}{31.40} & 42.96 & \colorbox{medalsilver}{0.526} & 2.43$\times$ & \colorbox{white}{35.80} & 42.06 & \colorbox{medalgold}{0.899} & 7.42$\times$ \\
& Original DLM & \colorbox{medalsilver}{36.00} & 128.00 & \colorbox{white}{1.281} & 1.00$\times$ & \colorbox{medalgold}{43.20} & 384.00 & \colorbox{white}{6.675} & 1.00$\times$ \\
\midrule
GSM8K & TDLM & \colorbox{medalgold}{73.77} & 26.67 & \colorbox{medalgold}{0.282} & 4.33$\times$ & \colorbox{medalbronze}{80.82} & 48.32 & \colorbox{medalsilver}{0.898} & 7.09$\times$ \\
& Fast-dLLM & \colorbox{medalbronze}{73.09} & 52.85 & \colorbox{medalbronze}{0.520} & 2.35$\times$ & \colorbox{white}{80.74} & 90.62 & \colorbox{medalbronze}{1.557} & 4.09$\times$ \\
& MetaState & \colorbox{white}{71.04} & 53.25 & \colorbox{white}{0.580} & 2.11$\times$ & \colorbox{medalgold}{81.27} & 93.51 & \colorbox{white}{1.700} & 3.75$\times$ \\
& CDLM & \colorbox{white}{53.83} & 40.54 & \colorbox{white}{0.820} & 1.49$\times$ & \colorbox{white}{68.84} & 72.06 & \colorbox{white}{1.562} & 4.08$\times$ \\
& d3LLM & \colorbox{white}{66.41} & 37.82 & \colorbox{medalsilver}{0.388} & 3.15$\times$ & \colorbox{white}{78.85} & 36.21 & \colorbox{medalgold}{0.573} & 11.12$\times$ \\
& Original DLM & \colorbox{medalsilver}{73.16} & 128.00 & \colorbox{white}{1.223} & 1.00$\times$ & \colorbox{medalsilver}{81.20} & 384.00 & \colorbox{white}{6.370} & 1.00$\times$ \\
\midrule
MBPP & TDLM & \colorbox{medalgold}{45.91} & 23.82 & \colorbox{medalgold}{0.233} & 4.84$\times$ & \colorbox{medalsilver}{45.91} & 40.19 & \colorbox{medalsilver}{0.725} & 8.54$\times$ \\
& Fast-dLLM & \colorbox{medalsilver}{45.53} & 26.96 & \colorbox{medalbronze}{0.242} & 4.67$\times$ & \colorbox{medalgold}{46.30} & 51.97 & \colorbox{white}{0.864} & 7.16$\times$ \\
& MetaState & \colorbox{medalsilver}{45.53} & 30.09 & \colorbox{white}{0.305} & 3.71$\times$ & \colorbox{medalbronze}{45.53} & 66.31 & \colorbox{white}{1.178} & 5.26$\times$ \\
& CDLM & \colorbox{medalbronze}{45.14} & 25.05 & \colorbox{white}{0.511} & 2.21$\times$ & \colorbox{medalsilver}{45.91} & 35.03 & \colorbox{medalbronze}{0.747} & 8.29$\times$ \\
& d3LLM & \colorbox{medalsilver}{45.53} & 21.02 & \colorbox{medalsilver}{0.237} & 4.77$\times$ & \colorbox{white}{43.58} & 28.26 & \colorbox{medalgold}{0.610} & 10.15$\times$ \\
& Original DLM & \colorbox{white}{44.75} & 128.00 & \colorbox{white}{1.130} & 1.00$\times$ & \colorbox{medalbronze}{45.53} & 384.00 & \colorbox{white}{6.191} & 1.00$\times$ \\
\bottomrule
\end{tabular}
\caption{\textbf{LLaDA with 128- and 384-token generation budgets at block size 32.} ACC is mathematical accuracy or code pass@1. NFE and seconds are per-problem means. Speedup is relative to original DLM on the same task and at the same budget. Medal colors and ties follow Table~\ref{tab:results256} within each dataset and budget.}
\label{tab:llada-generation-budgets}
\end{table}

\end{document}